\documentclass[letterpaper]{article} %
\usepackage[preprint]{aaai2027}  %
\usepackage{times}  %
\usepackage{helvet}  %
\usepackage{courier}  %
\usepackage[hyphens]{url}  %
\usepackage{graphicx} %
\usepackage{natbib}  %
\usepackage{caption} %
\usepackage{algorithm}
\usepackage{algorithmic}

\usepackage{booktabs,multirow,multicol}
\usepackage{graphicx}
\usepackage[table]{xcolor}
\usepackage{pifont}
\usepackage{bm}
\usepackage{amssymb}
\usepackage{amsthm,amsmath}
\newtheorem{theorem}{Theorem}
\newtheorem{manualtheoreminner}{Theorem}
\newenvironment{manualtheorem}[1]
  {\renewcommand\themanualtheoreminner{#1}%
   \manualtheoreminner}
  {\endmanualtheoreminner}

\theoremstyle{definition}
\newtheorem{definition}{Definition}

\newcommand{\vx}{\bm{x}}

\usepackage{newfloat}
\usepackage{listings}
\DeclareCaptionStyle{ruled}{labelfont=normalfont,labelsep=colon,strut=off} %
\floatstyle{ruled}
\newfloat{listing}{tb}{lst}{}
\floatname{listing}{Listing}
\title{Controllable Stochastic Quantization Encoding for \\ Adversarially Robust Spiking Neural Networks}
\author{
    Yujia Liu\textsuperscript{\rm 1,\rm 2},
    Peiyu Liu\textsuperscript{\rm 2, \rm 3},
    Yajing Zheng\textsuperscript{\rm 1,\rm 2},
    Tiejun Huang\textsuperscript{\rm 1,\rm 2}
}
\affiliations{
    \textsuperscript{\rm 1}NERCVT, School of Computer Science, Peking University, China\\
    \textsuperscript{\rm 2}State Key Laboratory for Multimedia Information Processing, School of Computer Science, Peking University, China
    \textsuperscript{\rm 3}Institute for Artificial Intelligence, Peking University, China
}

\usepackage{bibentry}

\begin{document}

\maketitle

\begin{abstract}
Spiking Neural Networks (SNNs) have attracted increasing attention due to their impressive temporal dynamics, energy efficiency, and brain-inspired mechanisms. Although SNNs have demonstrated promising performance in image classification tasks, recent studies have shown that they remain vulnerable to adversarial attacks, where imperceptible perturbations are added to input images to mislead model predictions. Existing defense methods mainly focus on training strategies, while the role of input encoding remains less explored.
An observation is that the robustness advantage of Poisson encoding over direct encoding may benefit from its inherent randomness. Motivated by this, we propose a stochastic quantization encoding method that encodes the input image with controllable randomness adjusted by the quantization scale, thereby improving the adversarial robustness of SNNs. We further show that this method constitutes a general framework that reduces to both Poisson encoding and direct encoding under different choices of the quantization scale. Since it enhances robustness at the input encoding stage, it can be combined with existing training-based defenses for further gains.
Experimental results on CIFAR-10 and CIFAR-100 demonstrate the effectiveness of the proposed stochastic quantization encoding method.
To sum up, this work highlights the importance of input encoding for the adversarial robustness of SNNs, providing a new perspective for understanding and improving it.
\end{abstract}

\section{Introduction}

Spiking Neural Networks (SNNs) are characterized by high energy efficiency and brain-inspired computing mechanisms, making them an attractive topic for neuromorphic computing~\citep{maas1997networks,bohte2000spikeprop,zheng2021going}. Different from Artificial Neural Networks (ANNs), which process information through continuous real-valued activations~\citep{2016_CVPR_He_ResNet}, SNNs employ discrete binary spike trains for information transmission~\citep{roy2019towards}. In addition, SNNs explicitly incorporate temporal dynamics, allowing them to effectively capture and process temporal information~\citep{kim2023exploring,han2025exploring}. Because of these advantages, SNNs have demonstrated remarkable performance in image classification tasks, with some models achieving results comparable to those of ANNs~\citep{huang2024towards,bu2025inference,zhou2026spikingformer}.

However, similar to ANNs, SNNs are also vulnerable to adversarial attacks~\citep{sharmin2019comprehensive,bu2023rate}, where carefully crafted imperceptible perturbations can lead the network to give incorrect predictions~\citep{2015_ICLR_Goodfellow_FGSM,2017_Chen_ZOO_AISec}. To address this issue, a variety of defense methods have been proposed for SNNs. Among them, adversarial training and its variants are the most widely adopted approaches, which enhance robustness by augmenting the training set with adversarial examples and solving a min-max optimization problem~\citep{aleksander2018towards}. Other methods aim to improve robustness through specifically designed regularization terms in the training loss, such as enforcing a bounded Lipschitz constant~\citep{ding2022snn} or introducing temporal constraints~\citep{leontev2021robustness}. More recently, it has been reported that promoting gradient sparsity can also contribute to improved adversarial robustness~\citep{liu2024enhancing,TGO}.

\begin{figure*}[!t]
    \centering
    \includegraphics[width=0.9\textwidth]{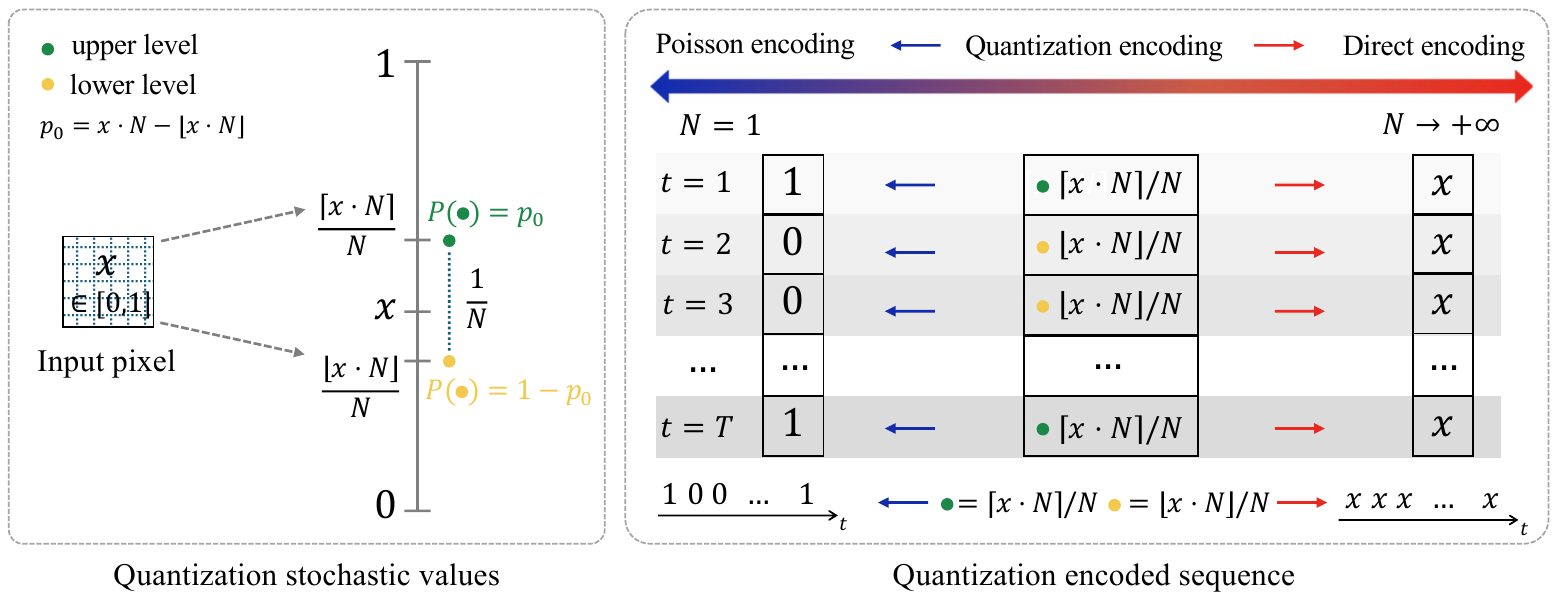}
    \caption{
    Illustration of the proposed stochastic quantization encoding method. At each time step, the normalized pixel value $x$ is randomly encoded as the upper quantiztaion level $\lceil xN\rceil/N$ with probability $p_0$, or as the lower quantiztaon level $\lfloor xN\rfloor/N$ with probability $1-p_0$, where $N$ is the quantization scale. At different time steps, the encoded values are independently sampled from the same distribution, and together form a quantization encoded sequence. When $N=1$, the proposed method reduces to Poisson encoding, where each encoded value is sampled from $\{0,1\}$. As $N$ increases, the encoded values become more stable and gradually approach direct encoding.
    }
    \label{fig:intro}
\end{figure*}

In addition to training strategies, input encoding methods also play a crucial role in the adversarial robustness of SNNs. Direct encoding and Poisson encoding are the most widely used methods in SNN training~\citep{wu2019direct,kim2022rate,fang2023spikingjelly}. For input images normalized to $[0,1]$, direct encoding feeds constant pixel intensities across all time steps, whereas Poisson encoding converts each pixel into stochastic spike trains where the pixel value serves as the firing probability. Recent studies have shown that Poisson encoding generally improves adversarial robustness compared with direct encoding, but at the cost of a noticeable drop in clean-image classification accuracy~\citep{sharmin2020inherent}. 
This phenomenon is closely related to the randomness introduced by Poisson encoding~\citep{RSC}. Stochastic spike generation may weaken the effect of adversarial perturbations, but it can also inject uncertainty into clean-image representations, leading to less stable predictions.

Inspired by this observation, we propose a stochastic quantization encoding method that can control the degree of randomness in input representations, thereby improving adversarial robustness while keeping the degradation in clean accuracy under control. As shown in Fig.~\ref{fig:intro}, each input pixel value $x$ is scaled by a quantization scale $N$ and then stochastically rounded to either $\lfloor xN \rfloor / N$ or $\lceil xN \rceil / N$ according to its fractional part. The quantization scale $N$ controls the upper bound of the encoding randomness, so different choices of $N$ achieve different trade-offs between adversarial robustness and clean accuracy. We further show that the proposed encoding provides a unified formulation of direct encoding and Poisson encoding under different choices of $N$.
Since the proposed method operates at the input encoding stage, it can be combined with existing training-based defense strategies to further enhance adversarial robustness.
The main contributions of this work are summarized as follows:

\begin{itemize}
\setlength{\itemsep}{0pt}
\setlength{\parsep}{0pt} 
\setlength{\parskip}{0pt}
    \item  We propose stochastic quantization encoding, an input encoding method for SNNs that introduces controllable randomness into input representations and thereby improves adversarial robustness.
    \item  We provide a theoretical analysis showing that the proposed quantization encoding is a general formulation that recovers both direct encoding and Poisson encoding as special cases under different choices of $N$.
    \item Experimental results on CIFAR-10 and CIFAR-100 validate the effectiveness of the proposed method and demonstrate its compatibility with existing training strategies for further robustness enhancement.
\end{itemize}

\section{Related Work}

\subsection{Adversarial Attacks on SNNs}
Adversarial attacks were first proposed in the context of ANNs. Recent studies have shown that SNNs are also susceptible to adversarial attacks~\citep{sharmin2020inherent,kang2026timestep}. Among existing attacks on SNNs, the Fast Gradient Sign Method (FGSM)~\citep{2015_ICLR_Goodfellow_FGSM}, Projected Gradient Descent (PGD)~\citep{aleksander2018towards}, and their variants are the most commonly used.

\textbf{FGSM} is a simple one-step attack method that generates adversarial examples by moving the input along the sign direction of the loss gradient, with the perturbation constrained by an $\ell_\infty$ norm bound~\citep{2015_ICLR_Goodfellow_FGSM}. \textbf{RFGSM} extends FGSM by applying a small random perturbation before the one-step gradient update, which helps move the input away from the non-smooth neighborhood of the original data point before linearizing the model loss~\citep{RFGSM}.

Different from the above attack methods, \textbf{PGD} is an iterative adversarial attack that generates adversarial examples by applying an FGSM-like update at each iteration, followed by a projection step that maps the adversarial examples back onto the $\epsilon$-neighborhood of the original image under the $\ell_\infty$ norm~\citep{aleksander2018towards}. 
\textbf{APGD} extends PGD by replacing the fixed step size with an adaptive step-size strategy~\citep{APGD}. Specifically, the step size is automatically adjusted according to the overall perturbation budget and the optimization progress.

\subsection{Defense Methods for SNNs}
To improve the robustness of SNNs against adversarial attacks, various defense methods have been proposed. 

A major line of research focuses on modifying the training process. Adversarial training~\citep{aleksander2018towards} and its variants~\citep{ding2022snn,wang2026boosting} are widely adopted in this direction. Typically, adversarial training~\citep{aleksander2018towards} improves robustness by solving a min-max optimization problem, in which adversarial examples are generated and incorporated into the training process. Building on this framework, \citet{ding2022snn} further regularize the Lipschitz constant of SNNs during adversarial training to enhance robustness. Other training-based defenses improve adversarial robustness by introducing carefully designed regularization terms into the loss function. For example, \citet{liu2024enhancing} propose a regularization term that promotes sparsity in the gradients of SNNs. The TGO method~\citep{TGO} introduces a membrane-potential constraint that pushes neurons' membrane potentials away from their firing thresholds, thereby reducing sensitivity to input perturbations and improve robustness.

Compared with methods that modify the training procedure, relatively few studies have explored SNN robustness from the perspective of input encoding. Poisson encoding~\citep{dayan2005theoretical,sharmin2020inherent} is a biologically inspired rate-based encoding scheme that converts input intensities into stochastic spike trains. Prior work has shown that Poisson encoding can achieve stronger robustness than direct encoding, though typically at the cost of a significant degradation in clean accuracy. Further analysis suggests that this accuracy--robustness trade-off is closely related to the inherent randomness of Poisson encoding~\citep{RSC}. Motivated by this observation, we propose a stochastic quantization encoding method to improve SNN robustness, where the upper bound of randomness can be controlled by adjusting the quantization scale.

\section{Preliminary}

\subsection{Neuron Dynamics in SNNs}

The most distinctive feature of neurons in SNNs, compared with those in ANNs, is that they incorporate temporal dynamics and transmit information through binary spikes. In this paper, we mainly consider the widely used Leaky Integrate-and-Fire (LIF) neuron model due to its efficiency and simplicity. The temporal dynamics of LIF neurons can be formulated as follows:
\begin{align}
\label{eq:LIF-model}
    u_i^l[t] &= \tau u_i^l[t-1](1-s_i^l[t-1]) + \sum_{j} w_{i,j}^{l-1} s_j^{l-1}[t], \\
\label{eq:Heaviside}
    s_i^l[t] & = H(u_i^l[t]-\theta).
\end{align}
Here, $u_i^l[t]$ denotes the membrane potential of the $i$-th neuron in layer $l$ at discrete time step $t$. As shown in Eq.(\ref{eq:LIF-model}), the membrane potential integrates both the decayed membrane potential from the previous time step and the weighted presynaptic spikes from the previous layer at the current time step. Specifically, the previous membrane potential $u_i^l[t-1]$ is scaled by the time constant $\tau$ and modulated by the reset term $(1-s_i^l[t-1])$, while the binary spikes $s_j^{l-1}[t]$ from the preceding layer are weighted by the synaptic weights $w_{i,j}^{l-1}$. According to Eq.(\ref{eq:Heaviside}), the output spike $s_i^l[t]$ is generated by the Heaviside step function $H(\cdot)$ when the membrane potential exceeds the firing threshold $\theta$.
It should be noted that layer $0$ denotes the input layer. The input image $\displaystyle \vx$ is encoded into a spike sequence, as will be introduced in the next subsection.

\subsection{Input Encoding Methods}

Direct encoding and Poisson encoding are two widely used input encoding methods in SNN training~\citep{fang2023spikingjelly}. In this paper, we assume pixel values of the input image $\displaystyle \vx$ are normalized to $[0,1]$, with each RGB value divided by $255$.

In direct encoding, the normalized pixel values are directly used as the input-layer activations at each time step:
\begin{equation}
    s_i^0[t] = x_i,
\end{equation}
where $x_i$ denotes the $i$-th normalized pixel value of $\displaystyle \vx$. Note that we slightly abuse notaion $s_i^0[t]$ here. Unlike the binary spikes $s_i^l[t]\in \{0,1\}$ for $l\geqslant 1$, the input spikes $s_i^0[t]$ in direct encoding take continuous values in $[0,1]$, as a deliberate relaxation at the input layer. In this way, the same input is provided to the network throughout all time steps.

For Poisson encoding, the normalized pixel value $x_i$ is treated as the firing probability, and the input spike $s_i^0[t]$ is independently sampled according to this probability at each time step, i.e.,
\begin{equation}
    P(s_i^0[t] = 1) = x_i, \qquad P(s_i^0[t] = 0) = 1 - x_i.
\end{equation}
Thus, each pixel value is represented by a stochastic binary spike train over time.

\section{Method}
In this paper, we propose a controllable Stochastic Quantization Encoding (SQE) method that encodes each input pixel into a sequence of random real values for SNN training. Specifically, we first describe how a single pixel is encoded into a stochastic quantization encoded sequence. We then define the randomness of an encoding method and show that the quantization scale controls the upper bound on the randomness of the proposed method. Furthermore, we demonstrate that the proposed method serves as a general encoding framework that connects Poisson encoding and direct encoding through different choices of the quantization scale. Finally, we provide a biological interpretation of the proposed method from the perspective of population coding.

\subsection{Stochastic Quantization Encoding Method}

The proposed SQE method introduces randomness through the quantization process, with the quantization scale controlling the quantization interval. 

To illustrate the encoding procedure, we take a single pixel as an example (Fig.~\ref{fig:intro}). Let $x_i\in [0,1]$ denote the normalized pixel value, and let $N\in \mathbb{Z}^+$ represent the quantization scale where $\mathbb{Z}^+$ is the set of positive integers. At each time step $t$, the proposed SQE method independently samples the encoded value $s_i^0[t]$ from the two adjacent quantization levels of $x_i$ when $x_iN \notin \mathbb{Z}$:
\begin{equation}
    s_i^0[t] \in \left\{
    \frac{\lfloor x_iN \rfloor}{N},
    \frac{\lceil x_iN \rceil}{N}
    \right\},
\end{equation}
where $\lfloor \cdot \rfloor$ and $\lceil \cdot \rceil$ denote the floor and ceiling operators, and $\frac{\lfloor x_iN \rfloor}{N}$ and $\frac{\lceil x_iN \rceil}{N}$ represent the lower and upper quantization levels, respectively.
The probabilities of selecting these two quantization levels are defined as
\begin{align}
    P\left(s_i^0[t]=\frac{\lceil x_iN \rceil}{N}\right) 
        &= x_iN - \lfloor x_iN \rfloor, \\
    P\left(s_i^0[t]=\frac{\lfloor x_iN \rfloor}{N}\right) 
        &= 1 - \left(x_iN - \lfloor x_iN \rfloor\right).
\end{align}

In other words, the probability of encoding $x_i$ into its upper quantization level is determined by the fractional part $\{x_iN\}$ of $x_iN$.
Repeating this independent sampling process over $T$ time steps produces a stochastic quantization encoded sequence $\{s_i^0[t]\}_{t=1}^{T}$.
Moreover, in the special case where $x_i N \in \mathbb{Z}$, we define $s_i^0[t] = x_i$ identically for all $t$.

\begin{theorem}[Unbiasedness]
\label{thr:unbias}
    The stochastic quantization encoding is unbiased at each time step,
    \begin{equation}
        \mathbb{E}\!\left[ s_i^0[t] \right] = x_i.
    \end{equation}
    and the time-averaged encoding converges to $x_i$ as $T\to \infty$, i.e.,
    \begin{equation}
        \frac{1}{T}\sum_{t=1}^{T} s_i^0[t] \xrightarrow{\;T \to \infty\;} x_i.
    \end{equation}
\end{theorem}
Therefore, the proposed SQE method preserves the input intensity in expectation. Due to the space limit, the proof of Theroem~\ref{thr:unbias} is shown in the Appendix A.

We further quantify the randomness of an encoding method by the variance of the encoded value at each time step. Since the encoding is sampled independently and identically across time steps, a larger variance indicates stronger temporal fluctuation of the encoded sequence, reflecting greater randomness in the encoding process.
\begin{definition} 
\label{def:randomness}
    The randomness of an encoding method at a single pixel $x_i$ is defined as
        \begin{equation}
         \operatorname{Var}\left[s^0_i[t]\right].
        \end{equation}
\end{definition}

For the proposed SQE method, the quantization scale $N$ determines the upper bound of the randomness for any pixel.
\begin{theorem}[Randomness]
\label{thr:random}
    For any pixel value $x_i\in[0,1]$, the randomness of the proposed stochastic quantization encoding method is bounded by
    \begin{equation}
        0 \leqslant \operatorname{Var}\left[s^0_i[t]\right] \leqslant \frac{1}{4N^2}.
    \end{equation}
\end{theorem}

The proof of the theorem is shown in the Appendix B.
This theorem indicates that a larger $N$ leads to a tighter upper bound on the randomness, making the encoded sequence more stable. Consequently, increasing $N$ would improve clean accuracy by producing more stable encoded inputs. However, it may also weaken adversarial robustness, since it closely preserves the perturbed pixel values, allowing adversarial perturbations to be transmitted to the network. Conversely, a smaller $N$ loosens the unpper bound on the randomness, making the encoded sequence more stochastic, which may benifit the robustness at the cost of the clean accuracy. The ablation study of $N$ in Sec.~\ref{sec:exp_ablation} empirically supports this point. Therefore, an appropriate $N$ can be chosen to control the randomness of the encoded sequence and balcance the clean accuracy against adversarial robustness.

Since the proposed SQE method involves a discrete sampling operation, it is not directly differentiable. During training, we employ the Straight-Through Estimator (STE) to approximate the gradients~\citep{STE}, where the forward pass performs stochastic quantization while the backward pass directly propagates the upstream gradients through the encoding layer. Details are provided in Appendix C.

\subsection{Analysis of the Quantization Scale }

In this section, we show that the proposed SQE method can be regarded as a general encoding framework that connects Poisson encoding and direct encoding through the quantization scale $N$.

\begin{theorem}[Poisson Encoding]
\label{thr:n=1}
    When $N=1$, the proposed stochastic quantization encoding method reduces to Poisson encoding.
\end{theorem}

\begin{theorem}[Direct Encoding]
\label{thr:n=infty}
    When $N\to+\infty$, the upper bound of the randomness in the proposed stochastic quantization encoding method approaches zero, and the method reduces to direct encoding.
\end{theorem}

The formal proofs of the above theorems are provided in the Appendix D. Here, we give an intuitive explanation. As shown in the left part of Fig.~\ref{fig:intro}, the distance between the upper and lower quantization levels is $1/N$. When $N=1$, the upper quantization level is $1$, while the lower quantization level is $0$. Therefore, the proposed method encodes each pixel into a sequence taking values from $\{0,1\}$, which has the same form as Poisson encoding. When $N\to+\infty$, the distance between the upper and lower quantization levels becomes extremely small. Consequently, both quantization levels become arbitrarily close to the original pixel value $x_i$, and the encoded sequence becomes nearly deterministic. In this case, the proposed method approaches direct encoding.

\subsection{A Biological Perspective}

It is interesting to find that SQE exhibits a behavior that is similar to population coding in the brain. Population coding~\citep{population_code} refers to a coding scheme in which information is represented collectively by a group of neurons, rather than relying on highly selective responses from individual neurons. One important property of population coding is its representational stability~\citep{population_code_size}. A larger neuronal population can more effectively average out trial-to-trial fluctuations in individual neuronal responses, resulting in a more stable representation, whereas a smaller population tends to produce a more random representation.

This offers a biological perspective on our method, in which the quantization scale plays a role analogous to the population size. The encoded result can be regarded as the information represented by the population coding. A larger quantization scale suppresses the randomness of the encoded values, driving them closer to the original input value, whereas a smaller scale admits stronger stochastic fluctuations and thus produces a more random representation.

\section{Experiments}

\begin{table*}[!t]
\caption{Effectiveness of the proposed SQE on the CIFAR-10 dataset under different training strategies, measured by classification accuracy (\%). The last column reports the average accuracy across all attacks. Better results are highlighted in gray.}
\centering
\resizebox{0.8\textwidth}{!}{
\begin{tabular}{cccllllllll}
\toprule
\multicolumn{1}{l}{Architecture}     & \multicolumn{1}{l}{Training}   & Encoding & Clean  & RFGSM  & PGD10  & PGD30  & PGD50  & APGD10 & Average \\
\midrule
\multirow{10}{*}{VGG-11} & \multirow{2}{*}{Vanilla}   & direct   & \cellcolor{gray!25}90.85 & 0.02  & 0.00  & 0.00  & 0.00  & 0.00    & 0.004    \\
                         &                            & SQE   & 87.36 & \cellcolor{gray!25}9.53  & \cellcolor{gray!25}9.43  & \cellcolor{gray!25}6.75  & \cellcolor{gray!25}6.36  & \cellcolor{gray!25}16.82 & \cellcolor{gray!25}9.78 \\
                         \cmidrule(lr){2-10}
                         & \multirow{2}{*}{AT}        & direct   & \cellcolor{gray!25}89.81 & 18.07 & 17.18 & 14.46 & 14.06 & 14.60 & 15.67 \\
                         &                            & SQE   & 88.51 & \cellcolor{gray!25}25.20 & \cellcolor{gray!25}24.94 & \cellcolor{gray!25}21.05 & \cellcolor{gray!25}20.80 & \cellcolor{gray!25}34.37 & \cellcolor{gray!25}25.27 \\
                         \cmidrule(lr){2-10}
                         & \multirow{2}{*}{RAT}       & direct   & \cellcolor{gray!25}90.44 & 19.56 & 19.05 & 13.84 & 13.15 & 17.31 & 16.58 \\
                         &                            & SQE   & 87.42 & \cellcolor{gray!25}30.09 & \cellcolor{gray!25}30.22 & \cellcolor{gray!25}24.77 & \cellcolor{gray!25}24.05 & \cellcolor{gray!25}43.41 & \cellcolor{gray!25}30.51 \\
                         \cmidrule(lr){2-10}
                         & \multirow{2}{*}{SR}        & direct   & \cellcolor{gray!25}89.28 & 12.30 & 11.71 & 8.45  & 8.08  & 9.35  & 9.98  \\
                         &                            & SQE   & 86.10 & \cellcolor{gray!25}25.79 & \cellcolor{gray!25}25.50 & \cellcolor{gray!25}22.66 & \cellcolor{gray!25}21.98 & \cellcolor{gray!25}34.20 & \cellcolor{gray!25}26.03 \\
                         \cmidrule(lr){2-10}
                         & \multirow{2}{*}{TGO}       & direct   & \cellcolor{gray!25}90.25 & 0.14  & 0.12  & 0.02  & 0.01  & 0.16  & 0.09  \\
                         &                            & SQE   & 85.69 & \cellcolor{gray!25}15.66 & \cellcolor{gray!25}15.52 & \cellcolor{gray!25}10.56 & \cellcolor{gray!25}9.84  & \cellcolor{gray!25}22.15 & \cellcolor{gray!25}14.75 \\
\midrule
\multirow{10}{*}{WRN-16} & \multirow{2}{*}{Vanilla}   & direct   & \cellcolor{gray!25}94.05 & 0.00  & 0.00  & 0.00  & 0.00  & 10.00    & 2.00    \\
                         &                            & SQE   & 88.95 & \cellcolor{gray!25}9.20  & \cellcolor{gray!25}9.19  & \cellcolor{gray!25}6.09  & \cellcolor{gray!25}6.06  & \cellcolor{gray!25}18.63 & \cellcolor{gray!25}9.83 \\
                         \cmidrule(lr){2-10}
                         & \multirow{2}{*}{AT}        & direct   & \cellcolor{gray!25}91.18 & 18.87 & 18.64 & 15.54 & 15.17 & 16.04 & 16.85 \\
                         &                            & SQE   & 89.03 & \cellcolor{gray!25}22.74 & \cellcolor{gray!25}22.43 & \cellcolor{gray!25}18.91 & \cellcolor{gray!25}18.18 & \cellcolor{gray!25}33.32 & \cellcolor{gray!25}23.12 \\
                         \cmidrule(lr){2-10}
                         & \multirow{2}{*}{RAT}       & direct   & \cellcolor{gray!25}92.14 & 21.14 & 20.45 & 14.05 & 13.12 & 19.26 & 17.60 \\
                         &                            & SQE   & 88.91 & \cellcolor{gray!25}34.22 & \cellcolor{gray!25}33.71 & \cellcolor{gray!25}27.37 & \cellcolor{gray!25}26.63 & \cellcolor{gray!25}48.43 & \cellcolor{gray!25}34.07 \\
                         \cmidrule(lr){2-10}
                         & \multirow{2}{*}{SR}        & direct   & \cellcolor{gray!25}92.62 & 0.92  & 0.86  & 0.29  & 0.25  & 0.68  & 0.60  \\
                         &                            & SQE   & 85.56 & \cellcolor{gray!25}23.59 & \cellcolor{gray!25}23.72 & \cellcolor{gray!25}20.33 & \cellcolor{gray!25}19.89 & \cellcolor{gray!25}32.51 & \cellcolor{gray!25}24.01 \\
                         \cmidrule(lr){2-10}
                         & \multirow{2}{*}{TGO}       & direct   & \cellcolor{gray!25}89.00 & 0.18  & 0.18  & 0.00  & 0.00  & 0.16  & 0.10  \\
                         &                            & SQE   & 83.24 & \cellcolor{gray!25}28.64 & \cellcolor{gray!25}28.64 & \cellcolor{gray!25}22.08 & \cellcolor{gray!25}21.34 & \cellcolor{gray!25}39.85 & \cellcolor{gray!25}28.11 \\
\bottomrule
\end{tabular}
}
\label{tab:cifar10_main}
\end{table*}

We evaluate the effectiveness of SQE on the CIFAR-10 and CIFAR-100 datasets. On the CIFAR-10 dataset, experiments are conducted using both VGG-11 and WideResNet-16 (WRN-16) based SNNs, while on CIFAR-100 only WRN-16 is used~\citep{simonyan2014very,zagoruyko2016wide}. For all experiments, the time step is fixed as $T=8$. Both the network architectures and training protocols follow those used in prior works~\citep{ding2022snn,liu2024enhancing,TGO} to ensure a fair comparison.

To evaluate adversarial robustness, we consider several popular attacks, including RFGSM~\citep{wong2020fast}, PGD~\citep{aleksander2018towards}, and APGD~\citep{APGD}. 
The perturbation budget is set to $8/255$ for all attacks. For iterative attacks, we indicate the number of iterations in the attack name. For example, PGD10 denotes PGD with 10 iterations.
Since the attack effectiveness may depend on the choice of gradient approximation methods and surrogate functions, we follow the evaluation strategy adopted in the prior work~\citep{liu2024enhancing}. Specifically, a set of surrogate functions and evaluate both STBP-based and RGA-based attacks~\citep{bu2023rate} are considered. During robustness evaluation, we test all combinations of attack settings and report the results under the strongest attack.

\subsection{Robustness Improvement under Different Training Strategies}
\label{sec:exp_main}

We evaluate the proposed method with $N=2$ on the CIFAR-10 dataset and $N=3$ for models trained on the CIFAR-100 dataset. As for the training strategy, we consider Adversarial Training (AT)~\citep{kundu2021hire} and its variant Regularized Adversarial Training (RAT)~\citep{ding2022snn}. We further include methods based on specially designed regularization terms, namely Sparse Regularization (SR)~\citep{liu2024enhancing} and Threshold Guarding Optimization (TGO)~\citep{TGO}. By combining the proposed SQE method with these representative defense strategies, we investigate its compatibility with existing training frameworks and its ability to further improve adversarial robustness. The implementation details of different training strategies are provided in the Appendix E. 

Tabs.~\ref{tab:cifar10_main} and \ref{tab:cifar100_main} report the classification accuracies on the CIFAR-10 and CIFAR-100 datasets, respectively, under both clean and adversarial conditions. In both tables, the "direct" in the third column denotes the baseline direct encoding, while the "SQE" denotes the proposed SQE method.

\begin{table*}[t]
\caption{Effectiveness of the proposed SQE on the CIFAR-100 dataset under different training strategies, measured by classification accuracy (\%). The last column reports the average accuracy over all attacks. Better results are highlighted in gray.}
\centering
\resizebox{0.8\textwidth}{!}{
\begin{tabular}{ccclllllll}
\toprule
\multicolumn{1}{l}{Architecture}     & \multicolumn{1}{l}{Training} & Encoding & clean  & RFGSM  & PGD10  & PGD30  & PGD50  & APGD10 & Average \\
\midrule
\multirow{10}{*}{WRN-16} & \multirow{2}{*}{Vanilla}     & direct   & \cellcolor{gray!25}73.44 & 0.00  & 0.00  & 0.00  & 0.00  & 0.00  & 0.00  \\
                         &                              & SQE   & 67.82 & \cellcolor{gray!25}2.93  & \cellcolor{gray!25}2.79  & \cellcolor{gray!25}1.74  & \cellcolor{gray!25}1.61  & \cellcolor{gray!25}9.04  & \cellcolor{gray!25}3.62  \\
                         \cmidrule(lr){2-10}
                         & \multirow{2}{*}{AT}          & direct   & \cellcolor{gray!25}67.99 & 11.22 & 10.92 & 8.92  & 8.30  & 9.25  & 9.72  \\
                         &                              & SQE   & 65.59 & \cellcolor{gray!25}11.97 & \cellcolor{gray!25}11.82 & \cellcolor{gray!25}9.60  & \cellcolor{gray!25}9.18  & \cellcolor{gray!25}19.63 & \cellcolor{gray!25}12.44 \\
                         \cmidrule(lr){2-10}
                         & \multirow{2}{*}{RAT}         & direct   & \cellcolor{gray!25}68.16 & 12.04 & 11.62 & 7.74  & 7.10  & 10.80 & 9.86  \\
                         &                              & SQE   & 65.22 & \cellcolor{gray!25}17.89 & \cellcolor{gray!25}17.96 & \cellcolor{gray!25}13.35  & \cellcolor{gray!25}13.04  & \cellcolor{gray!25}27.34 & \cellcolor{gray!25}17.92 \\
                         \cmidrule(lr){2-10}
                         & \multirow{2}{*}{SR}          & direct   & \cellcolor{gray!25}68.83 & 1.29  & 1.25  & 0.74  & 0.67  & 1.18  & 1.03  \\
                         &                              & SQE   & 61.24 & \cellcolor{gray!25}12.02 & \cellcolor{gray!25}11.85 & \cellcolor{gray!25}10.11 & \cellcolor{gray!25}9.75  & \cellcolor{gray!25}17.52 & \cellcolor{gray!25}12.25 \\
                         \cmidrule(lr){2-10}
                         & \multirow{2}{*}{TGO}         & direct   & \cellcolor{gray!25}69.22 & 0.27  & 0.22  & 0.03  & 0.03  & 0.30  & 0.17  \\
                         &                              & SQE   & 60.72 & \cellcolor{gray!25}12.41 & \cellcolor{gray!25}12.58 & \cellcolor{gray!25}8.00  & \cellcolor{gray!25}7.46  & \cellcolor{gray!25}19.64 & \cellcolor{gray!25}12.02 \\
\bottomrule
\end{tabular}
}
\label{tab:cifar100_main}
\end{table*}

From Tab.~\ref{tab:cifar10_main}, it can be observed that the proposed SQE method consistently improves adversarial robustness across different attack settings, while causing a limited reduction in clean accuracy. This trend is particularly evident for VGG-11. For instance, compared with vanilla SNNs using direct encoding, the proposed quantization encoding causes about a 3\% drop in clean accuracy, but achieves nearly a 10\% improvement in robustness against RFGSM and PGD10.
When combined with AT, the proposed encoding improves the accuracy under APGD10 by approximately 20\% compared with AT alone, while reducing the clean accuracy by only about 1\%. Similar observations can be made for the recently proposed SR strategy. By incorporating the proposed quantizatin encoding method, the accuracy under strong attacks such as PGD30 and PGD50 is further improved by more than 10\%, at the cost of only around a 3\% decrease in clean accuracy. As for WRN-16, the RAT-trained model equipped with the proposed quantization encoding achieves an average accuracy improvement of approximately 17\% across all attacks compared with RAT alone, at the cost of less than a 5\% reduction in clean accuracy.

As shown in Tab.~\ref{tab:cifar100_main}, the proposed SQE method generally improves adversarial accuracy on CIFAR-100 across different training strategies. Compared with direct encoding, quantization encoding sacrifices a certain amount of clean accuracy, which is attributed to the randomness introduced during the quantization process. Nevertheless, such randomness brings clear benefits under adversarial settings. For instance, the classification accuracy increases from 0.00\% to 9.04\% under vanilla training against APGD10. When combined with existing defense strategies, the gains are more evident. The average classificaiton accuracy against attacks improves from 9.86\% to 17.92\% under the RAT strategy, from 1.03\% to 12.25\% under SR strategy, and from 0.17\% to 12.02\% under TGO strategy. 

These results suggest that SQE method can serve as an effective and complementary input-level strategy for improving the adversarial robustness of SNNs.

\subsection{Albation Study on Quantiztaion Scale}
\label{sec:exp_ablation}
We conduct the ablaiton study of the quantiztaion scale $N$ on the CIFAR-10 dataset with the normal training strategy. The models are with the VGG-11 architecture. Specifically, we train the SNNs with differenet choices of $N$ in the set $\{2, 4, 8, 16\}$. The models are evaluated in terms of both clean accuracy and adversarial robustness under different attacks. The results are shown in Fig.~\ref{fig:ablation_N} where $N=1$ and $N=+\infty$ represents the Poisson encoding and direct encoding respectively according to Theorems~\ref{thr:n=1} and~\ref{thr:n=infty}.

\begin{figure}[h]
    \centering
    \includegraphics[width=0.8\linewidth]{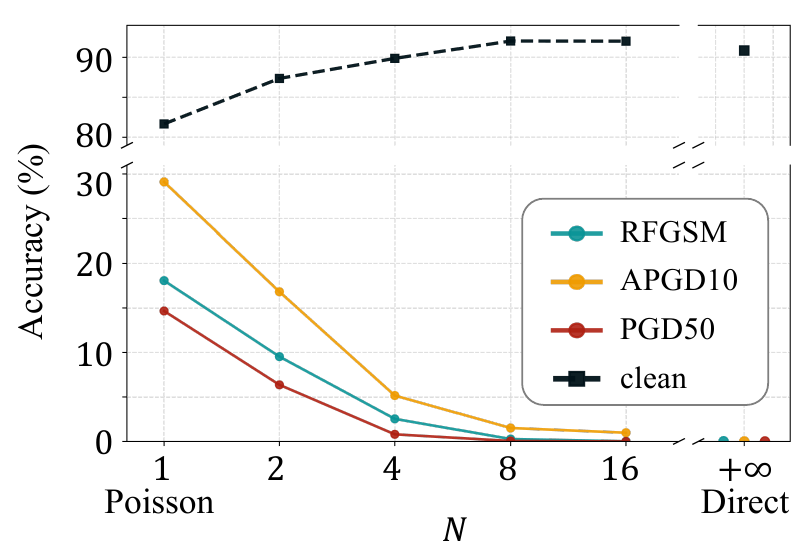}
    \caption{Ablation study of the quantization scale $N$ on the CIFAR-10 dataset with the VGG-11 arhitectue.}
    \label{fig:ablation_N}
\end{figure}

The experimental results are consistent with our theoretical analysis, revealing a clear trade-off controlled by the quantization scale $N$. As $N$ increases, the encoded representation becomes more deterministic, resulting in higher clean accuracy but reduced adversarial robustness. Conversely, smaller values of $N$ introduce stronger randomness, which improves robustness at the cost of clean accuracy. For example, increasing $N$ from 2 to 8 improves the clean accuracy by approximately 5\%, while causing the accuracies under all considered attacks to decrease to nearly zero.

\subsection{Comparsion with Poisson Encoding}
\label{sec:exp_poisson}

We further compare the proposed SQE method with direct and Poisson encoding under both white-box and black-box attacks to verify that it achieves a favorable trade-off between clean accuracy and adversarial robustness. We evaluate robustness under one-step attacks as well as iterative attacks. The results are reported in Tab.\ref{tab:direct_quant_poi}. For the black-box setting, adversarial examples are generated using surrogate models trained with direct encoding, which share the same architecture as the target models with different random seeds.

\begin{table*}[!t]
\caption{Comparison of direct encoding, Poisson encoding, and the proposed SQE method. The proposed SQE method is highlighted in gray, as it provides a more balanced trade-off between clean accuracy and adversarial robustness.}
\centering
\resizebox{0.8\textwidth}{!}{
\begin{tabular}{cllcccccc}
\toprule
\multicolumn{1}{l}{}                                                       &          &       & \multicolumn{3}{c}{white-box attack} & \multicolumn{3}{c}{black-box attack} \\
\cmidrule(lr){4-6} \cmidrule(lr){7-9}
\multicolumn{1}{l}{}                                                       & encoding & clean & RFGSM  & APGD10  & PGD50  & RFGSM  & APGD10  & PGD50  \\
\midrule
\multirow{3}{*}{\begin{tabular}[c]{@{}c@{}}CIFAR-10\\ VGG-11\end{tabular}}   & direct   & 90.85 & 0.02   & 0.00    & 0.00   & 8.23   & 18.84   & 7.38   \\
                                                                           & \cellcolor{gray!25}SQE & \cellcolor{gray!25}87.36 & \cellcolor{gray!25}9.53 & \cellcolor{gray!25}16.82 & \cellcolor{gray!25}6.36 & \cellcolor{gray!25}60.56 & \cellcolor{gray!25}74.95 & \cellcolor{gray!25}62.65 \\
                                                                           & poisson  & 81.67 & 18.07  & 29.11   & 14.66  & 61.97  & 76.15   & 62.89  \\
\midrule                                                                           
\multirow{3}{*}{\begin{tabular}[c]{@{}c@{}}CIFAR10\\ WRN-16\end{tabular}}   & direct   & 94.05 & 0.00   & 10.00   & 0.00   & 0.05   & 2.48    & 0.03   \\
                                                                           & \cellcolor{gray!25}SQE & \cellcolor{gray!25}88.95 & \cellcolor{gray!25}9.20 & \cellcolor{gray!25}18.63 & \cellcolor{gray!25}6.06 & \cellcolor{gray!25}72.01 & \cellcolor{gray!25}83.60 & \cellcolor{gray!25}72.98 \\
                                                                           & poisson  & 83.20 & 17.73  & 30.13   & 14.29  & 71.49  & 82.90   & 71.53  \\
\midrule                                                                           
\multirow{3}{*}{\begin{tabular}[c]{@{}c@{}}CIFAR-100\\ WRN-16\end{tabular}} & direct   & 73.44 & 0.00   & 0.00    & 0.00   & 0.84   & 9.39    & 0.42   \\
                                                                            & \cellcolor{gray!25}SQE & \cellcolor{gray!25}67.82 & \cellcolor{gray!25}2.93 & \cellcolor{gray!25}9.04 & \cellcolor{gray!25}1.61 & \cellcolor{gray!25}43.40 & \cellcolor{gray!25}59.47 & \cellcolor{gray!25}43.83 \\
                                                                           & poisson  & 54.07 & 10.34  & 16.55   & 7.93   & 39.13  & 52.57   & 39.53  \\
\bottomrule                                                                           
\end{tabular}
}
\label{tab:direct_quant_poi}
\end{table*}

\begin{figure}[!t]
    \centering
    \includegraphics[width=0.98\linewidth]{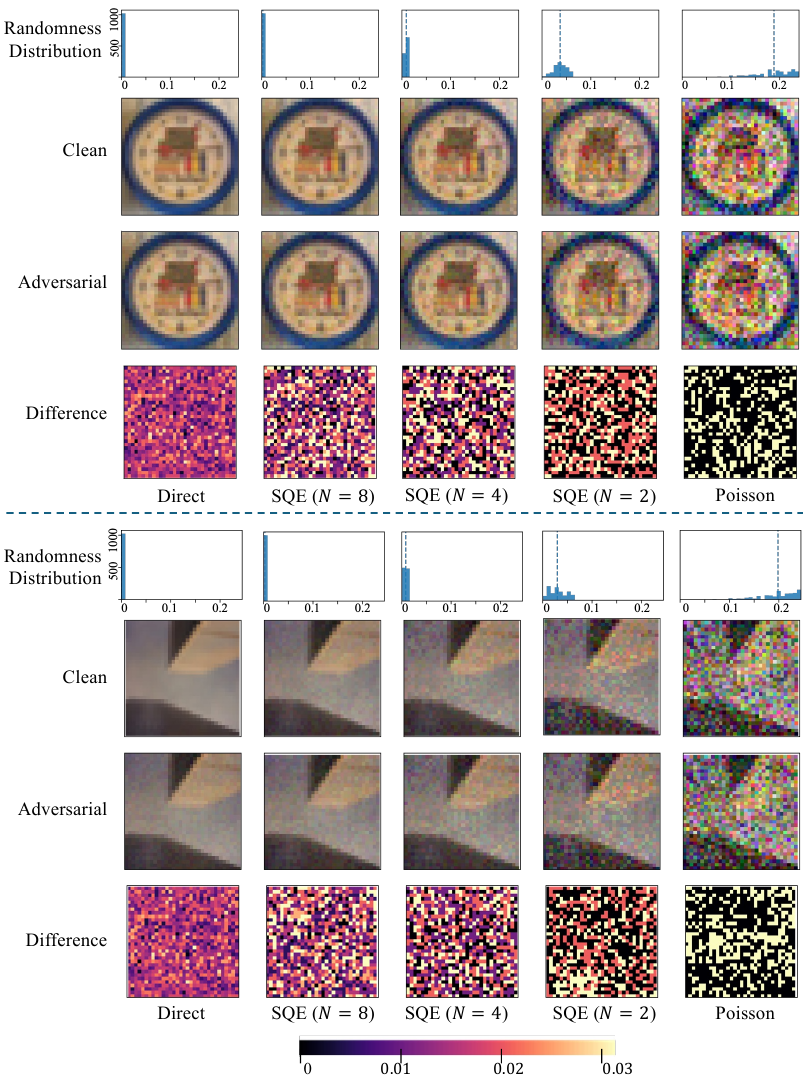}
    \caption{Visualization of clean and adversarial images (generated by PGD10) after different encoding methods. For SQE and Poisson encoding, the encoded results are averaged over 8 time steps. For each encoding method, the first row shows the distribution of pixel-wise encoding randomness (Definition~\ref{def:randomness}), with the mean randomness across all pixels indicated by the dashed line. The last row shows the difference between the encoded clean image and the encoded adversarial image.}
    \vspace{-10pt}
    \label{fig:visual}
\end{figure}

As expected, the proposed SQE method exhibits an intermediate behavior between the two baseline encoding schemes. Specifically, it incurs a much smaller loss in clean accuracy than Poisson encoding, while achieving substantially higher adversarial robustness than direct encoding. For example, in the white-box scenario, the proposed quantization encoding improves the accuracy under APGD10 from 0 to 16.82\% compared with direct encoding on CIFAR-10 VGG-11, while sacrificing only about 3\% of clean accuracy. Although Poisson encoding further improves the accuracy under APGD10 to 29.11\%, it reduces the clean accuracy to only 81.67\%, which is more than 9\% lower than that of direct encoding. 
Notably, in the black-box scenario, the proposed quantization encoding achieves comparable or even better robustness than Poisson encoding. 

These results suggest that an appropriate level of encoding randomness is crucial for achieving a favorable balance between clean accuracy and adversarial robustness. Direct encoding preserves the original image information and therefore achieves the highest clean accuracy, but offers almost no robustness against adversarial perturbations. In contrast, Poisson encoding introduces the strongest randomness, substantially improving robustness but also severely degrading the information fidelity of the encoded inputs, resulting in a significant loss of clean accuracy. The proposed SQE method provides a controllable level of randomness, effectively balancing information preservation and randomness injection. As a result, it achieves a more desirable trade-off between clean accuracy and adversarial robustness.

\subsection{Visulization}
\label{sec:exp_visual}

Fig.~\ref{fig:visual} presents two visual examples of images encoded using different encoding methods. As the quantization scale $N$ decreases, the histogram of pixel-wise encoding randomness gradually shifts to the right, indicating an increase in encoding randomness. or example, in the first example, the average encoding randomness is zero for direct encoding, increases to approximately $0.02$ as $N$ decreases from 8 to 2, and reaches about $0.19$ for Poisson encoding ($N=1$).
This observation confirms that the quantization scale $N$ effectively controls the degree of encoding randomness. Correspondingly, the difference between the encoded clean and adversarial images becomes increasingly sparse. Under direct encoding, the perturbation-induced differences are densely distributed over almost the entire image. As the encoding randomness increases, however, more pixels remain identical after encoding, and the perturbation-induced differences become progressively confined to fewer locations. Together with the quantitative results in Tab.~\ref{tab:direct_quant_poi}, these visualizations indicate that an appropriate level of encoding randomness helps achieve a better balance between preserving clean image information and reducing the effect of adversarial perturbations after encoding.

\section{Conclusion and Discussion}

In this paper, we propose the SQE method to improve the robustness of SNNs. By adjusting the quantization scale, the proposed method can control the trade-off between clean accuracy and adversarial robustness. 
This work offers a new perspective on adversarial robustness in SNNs by highlighting the important role of input encoding.
One limitation of SQE is that the quantization scale $N$ is fixed for all pixels in an image and is selected empirically, which limits its flexibility. In future work, we aim to explore learnable and pixel-adaptive quantization scales, 
which may enable a more effective trade-off between clean accuracy and adversarial robustness.

\bibliography{aaai2027}

\clearpage
\appendix

\section{Proof of Theroem 1}

\begin{manualtheorem}{1}[Unbiasedness]
    The stochastic quantization encoding is unbiased at each time step,
    \begin{equation}
        \mathbb{E}\left[ s^0_i[t] \right] = x_i.
    \end{equation}
    and the time-averaged encoding converges to $x_i$ as $T\to \infty$, i.e.,
    \begin{equation}
        \frac{1}{T}\sum_{t=1}^{T} s^0_i[t] \xrightarrow{\;T \to \infty\;} x_i.
    \end{equation}
\end{manualtheorem}

\begin{proof}
\mbox{}

\textbf{Case 1: $x_iN \in \mathbb{Z}$.}
   Here $s_i^0[t] = x_i$ is deterministic, and therefore $\mathbb{E}\!\left[ s_i^0[t] \right] = x_i$.

\textbf{Case 2: $x_iN \notin \mathbb{Z}$.} 
    For convenience, let
\begin{equation}
    p_0 = x_i\cdot N - \lfloor x_i\cdot N \rfloor.
\end{equation}
Then
\begin{align}
    \mathbb{E}\left[
        s_i^0[t]
    \right] 
    &= \frac{\lceil x_i\cdot N \rceil}{N} \cdot p_0 + \frac{\lfloor x_i\cdot N \rfloor}{N} \cdot (1 - p_0) \\
    &= \frac{\lceil x_i\cdot N \rceil}{N} \cdot p_0 + \frac{\lfloor x_i\cdot N \rfloor}{N} - \frac{\lfloor x_i\cdot N \rfloor}{N} \cdot p_0 \\
    &= \frac{\lceil x_i\cdot N \rceil - \lfloor x_i\cdot N \rfloor}{N} \cdot p_0 + \frac{\lfloor x_i\cdot N \rfloor}{N} \\
    &= \frac{1}{N}\cdot p_0 + \frac{\lfloor x_i\cdot N \rfloor}{N} \\
    &= \frac{1}{N}\cdot \left(x_i\cdot N - \lfloor x_i\cdot N \rfloor\right) + \frac{\lfloor x_i\cdot N \rfloor}{N} \\
    &= x_i.
\end{align}

Therefore, we have
\begin{equation}
        \mathbb{E}\!\left[ s_i^0[t] \right] = x_i.
    \end{equation}

Consequently, since $\{s_i^0[t]\}_{t=1}^{T}$ are independent and identically distributed, by the law of large numbers the sample mean of the encoded sequence converges to $x_i$ as $T \to \infty$:
    \begin{equation}
        \frac{1}{T}\sum_{t=1}^{T} s_i^0[t] \xrightarrow{\;T \to \infty\;} x_i.
    \end{equation}
\end{proof}

\section{Proof of Theorem 2}

\begin{manualtheorem}{2}[Randomness]
    For any pixel value $x_i\in[0,1]$, the randomness of the proposed stochastic quantization encoding method is bounded by
    \begin{equation}
    \label{eq:supp_random}
        0 \leqslant \operatorname{Var}\left[s_i^0[t]\right] \leqslant \frac{1}{4N^2}.
    \end{equation}
\end{manualtheorem}

\begin{proof}
\mbox{}

\textbf{Case 1: $x_iN \in \mathbb{Z}$.}
   Here $s_i^0[t] = x_i$ is deterministic, and therefore $\operatorname{Var}\left[s_i^0[t]\right] = 0$ which satisfies Eq.~(\ref{eq:supp_random}).

\textbf{Case 2: $x_iN \notin \mathbb{Z}$.}
    For convenience, let
        \begin{equation}
            p_0 = x_i\cdot N - \lfloor x_i\cdot N \rfloor.
        \end{equation}
    Then
        \begin{align}
            & \operatorname{Var}\left[s_i^0[t]\right]
            = \mathbb{E}[\left(s_i^0[t]\right)^2] - \mathbb{E}^2[s_i^0[t]] \\
            &= \left(\frac{\lceil x_i\cdot N \rceil}{N}\right)^2 \cdot p_0 + \left(\frac{\lfloor x_i\cdot N \rfloor}{N}\right)^2 \cdot (1-p_0) - x_i^2 \\
            &= \frac{\lceil x_i\cdot N \rceil^2 - \lfloor x_i\cdot N \rfloor^2}{N^2} \cdot p_0 + \frac{\lfloor x_i\cdot N \rfloor^2 }{N^2} - x_i^2\\
            &= \frac{\lceil x_i\cdot N \rceil + \lfloor x_i\cdot N \rfloor}{N^2} \cdot p_0 + \frac{\lfloor x_i\cdot N \rfloor^2}{N^2} - x_i^2\\
            &= \frac{\lceil x_i\cdot N \rceil}{N^2} \cdot p_0 + \frac{\lfloor x_i\cdot N \rfloor}{N^2}(p_0 + \lfloor x_i\cdot N \rfloor) - x_i^2\\
            &= \frac{\lceil x_i\cdot N \rceil}{N^2} \cdot p_0 + \frac{\lfloor x_i\cdot N \rfloor}{N^2} x_i\cdot N - x_i^2\\
            &= \frac{(\lfloor x_i\cdot N \rfloor + 1)(x_i\cdot N - \lfloor x_i\cdot N \rfloor)}{N^2} + \frac{\lfloor x_i\cdot N \rfloor}{N^2} x_i\cdot N - x_i^2 \\
            &= - \frac{\lfloor x_i\cdot N \rfloor^2 -2x_i\cdot N \lfloor x_i\cdot N \rfloor + x_i^2}{N^2} + \frac{x_i\cdot N - \lfloor x_i\cdot N \rfloor}{N^2} \\
            &= - \frac{\left( x_i - \lfloor x_i\cdot N \rfloor\right)^2}{N^2} + \frac{x_i\cdot N - \lfloor x_i\cdot N \rfloor}{N^2} \\
            &= - \frac{p_0^2}{N^2} + \frac{p_0}{N^2} \\
            &= \frac{p_0(1-p_0)}{N^2}.
        \end{align}
    Here, $p_0 \in [0,1]$, it follows that
    \begin{equation}
        0 \leqslant p_0(1-p_0) \leqslant 1/4.
    \end{equation}
    Thus, the randomness of the stochastic quantization encoding at each pixel is bounded by
    \begin{equation}
        0 \leqslant \operatorname{Var}\left[s_i^0[t]\right] \leqslant \frac{1}{4N^2}.
    \end{equation}
\end{proof}

\section{Backward Propagation}
Since the quantization operation involves discrete rounding and random sampling, it is non-differentiable and prevents gradients from being directly back-propagated through the encoding layer. To enable end-to-end training, we employ the Straight-Through Estimator (STE)~\citep{STE}, which retains the stochastic quantization operation in the forward pass while approximating it as an identity mapping in the backward pass. Specifically, the gradient of the loss $L$ with respect to the pixel value $x_i$ is approximated by
\begin{align}
    \frac{
        \partial L
    }{\partial x_i} &= \sum_{t=1}^ T \frac{\partial L}{\partial s_i^0[t]} \cdot \frac{\partial s_i^0[t]}{\partial x_i} \\
                    &\approx \sum_{t=1}^ T \frac{\partial L}{\partial s_i^0[t]}
\end{align}
where the STE approximates the non-differentiable term as
\begin{equation}
    \frac{\partial s_i^0[t]}{\partial x_i}
    \approx 1,
    \qquad t=1,\ldots,T.
\end{equation}
With this approximation, the gradients from all time steps can be directly back-propagated to the input pixel, enabling end-to-end optimization of the entire network while preserving the stochastic quantization behavior in the forward pass.

\section{Proof of Theorems 3 and 4}

\begin{manualtheorem}{3}[Poisson Encoding]
    When $N=1$, the proposed stochastic quantization encoding method reduces to Poisson encoding.
\end{manualtheorem}

\begin{proof}
\mbox{}

\textbf{Case 1: $x_i=0$.}
   According to the definition, we have $s_i^0[t]=0$ for all time steps. Therefore, the encoded value is deterministic.

\textbf{Case 2: $x_i=1$.} 
    According to the definition, $s_i^0[t]=1$ holds for all time steps. Thus, the encoded value is deterministic and always equal to $1$.

\textbf{Case 3: $0 < x_i <1$}
    Since $N=1$ and $0 < x_i < 1$, the upper and lower quantiztaion levels are given by
    \begin{align}
        \frac{\lfloor x_i\cdot N \rfloor}{N} &= \lfloor x_i \rfloor = 0, \\
        \frac{\lceil x_i\cdot N \rceil}{N} &= \lceil x_i \rceil = 1. \\
    \end{align}
    Moreover, the corresponding probabilities are
    \begin{align}
        P\left( s_i^0[t] = 1\right) &= x_i - \lfloor x_i \rfloor = x_i, \\
        P\left( s_i^0[t] = 0\right) &= 1 - x_i.
    \end{align}
    Thus, at each time step, $s_i^0[t]$ is sampled from $\{0,1\}$ with firing probability $x_i$, which is exactly the same as Poisson encoding.

To sum up, when $N=1$, the proposed stochastic quantization encoding method reduces to Poisson encoding.

\end{proof}

\begin{manualtheorem}{4}[Direct Encoding]
    When $N\to+\infty$, the upper bound of the randomness in the proposed stochastic quantization encoding method approaches zero, and the method reduces to direct encoding.
\end{manualtheorem}

\begin{proof}

For any $x_i\in[0,1]$, the stochastic quantization encoding maps $x$ to one of
its two adjacent quantization levels:
\begin{equation}
    s_i^0[t] \in 
    \left\{
    \frac{\lfloor x_iN \rfloor}{N},
    \frac{\lceil x_iN \rceil}{N}
    \right\}.
\end{equation}
By the definitions of the floor and ceiling operators, we have
\begin{equation}
    \lfloor x_iN \rfloor \leq x_iN \leq \lceil x_iN \rceil .
\end{equation}
Dividing both sides by $N$ gives
\begin{equation}
    \frac{\lfloor x_iN \rfloor}{N} \leq x_i \leq
    \frac{\lceil x_iN \rceil}{N}.
\end{equation}
Moreover, the distance between the two adjacent quantization levels is at most
$1/N$, since
\begin{equation}
    \frac{\lceil x_iN \rceil}{N} -
    \frac{\lfloor x_iN \rfloor}{N}
    \leq \frac{1}{N}.
\end{equation}
Therefore, regardless of which quantization level is sampled, we have
\begin{equation}
    \left| s_i^0[t] - x_i \right| \leq \frac{1}{N}.
\end{equation}
Taking the limit $N\to\infty$, we obtain
\begin{equation}
    \lim_{N\to\infty} \left| s_i^0[t] - x_i \right| = 0,
\end{equation}
which implies
\begin{equation}
    s_i^0[t] \to x_i .
\end{equation}
Thus, the proposed stochastic quantization encoding converges to direct
encoding as $N\to\infty$.

It is worth noting that, in the classification experiments considered in this paper, the normalized pixel value is discrete and can be written as $x_i=m/255$, where $m\in\{0,1,2,\ldots,255\}$. Therefore, setting $N=255$ is sufficient to make $x_iN$ an integer. In this case, the stochastic quantization encoding is deterministic, i.e., $s_i^0[t]=x_i$ for all time steps, and thus recovers direct encoding.

\end{proof}

\section{Training Details}
All experiments are conducted on the NVIDIA GeForce RTX 4090 GPU with 24 GB GDDR6X memory.
In this section, we provide the detailed training parameters for all training strategies. Some parameters, such as the learning rate and training epochs, are the same across different strategies, while others are specific to particular strategies, such as the weights of regularization terms. 

We first describe the common parameters across different training strategies. For all models trained on CIFAR-10 and CIFAR-100 datasets, we use an initial learning rate of 0.1, a momentum of 0.9, and a weight decay of $5 \times 10^{-4}$. All models are trained for 200 epochs with a batch size of 64, and the number of time steps is fixed at 8.

\begin{table}[htbp]
\caption{The common parameters for all models on CIFAR-10 and CIFAR-100 datasets.}
\centering
\resizebox{0.43\linewidth}{!}{
\begin{tabular}{ll}
\toprule
parameters  & value  \\
\midrule
learning rate & 0.1  \\
momentum      & 0.9  \\
weight decay  & 5e-4 \\
batch size    & 64   \\
epochs        & 200  \\
time step     & 8   \\
\bottomrule
\end{tabular}
}
\label{tab:common_setting}
\end{table}

We then describe the strategy-specific parameters for different training strategies. For adversarial training (AT), we use the PGD attack to generate adversarial examples during training. The perturbation budget is set to $2.0/255$, and the number of attack iterations is set to 5. For the RAT strategy, we use the same attack settings as AT. Following the authors' recommendation~\cite{ding2022snn}, we adopt the BPTR attack mode and set the weight of the regularization term to 0.001 for VGG-11 and 0.004 for WRN-16. For the SR strategy, the finite-difference step size is set to 0.01, and the sparsity regularization weight is set to 0.002 for all models. For the TGO strategy, $\lambda_{\max}$ is set to 0.4 for all models. The parameters are summarized as follows.

\begin{table}[htbp]
\caption{Strategy-specific parameter settings for different training strategies.}
\centering
\resizebox{0.75\linewidth}{!}{
\begin{tabular}{cllc}
\toprule
\multicolumn{1}{l}{strategy} & parameter               & value & models                             \\
\midrule
\multirow{3}{*}[-0.8ex]{AT}   & attack                  & PGD   & \multirow{3}{*}[-0.8ex]{All}               \\ \cmidrule(lr){2-3}
                     & $\epsilon$              & 2.0 / 255     &                                    \\ \cmidrule(lr){2-3}
                     & iteration               & 5     &                                    \\ 
\midrule
\multirow{6}{*}[-2.5ex]{RAT} & attack                  & PGD   & \multirow{4}{*}[-1.4ex]{All}               \\ \cmidrule(lr){2-3}
                     & $\epsilon$              & 2.0 / 255     &                                    \\ \cmidrule(lr){2-3}
                     & iteration               & 5     &                                    \\ \cmidrule(lr){2-3}
                     & attack mode             & bptr  &                                    \\ \cmidrule(lr){2-4}
                     & \multirow{2}{*}{weight} & 0.001 & VGG-11                             \\ \cmidrule(lr){3-4}
                     &                         & 0.004 & WRN-16                             \\ 
\midrule
\multirow{2}{*}[-0.6ex]{SR}  & step size       & 0.01  & \multirow{2}{*}[-0.6ex]{All}       \\ \cmidrule(lr){2-3}
                     & weight                  & 0.002 &                                    \\
\midrule
TGO                  & $\lambda_{max}$         & 0.4   & All                                \\
\bottomrule
\end{tabular}
}
\label{tab:specific_setting}
\end{table}

\begin{figure*}[!t]
    \centering
    \includegraphics[width=0.8\textwidth]{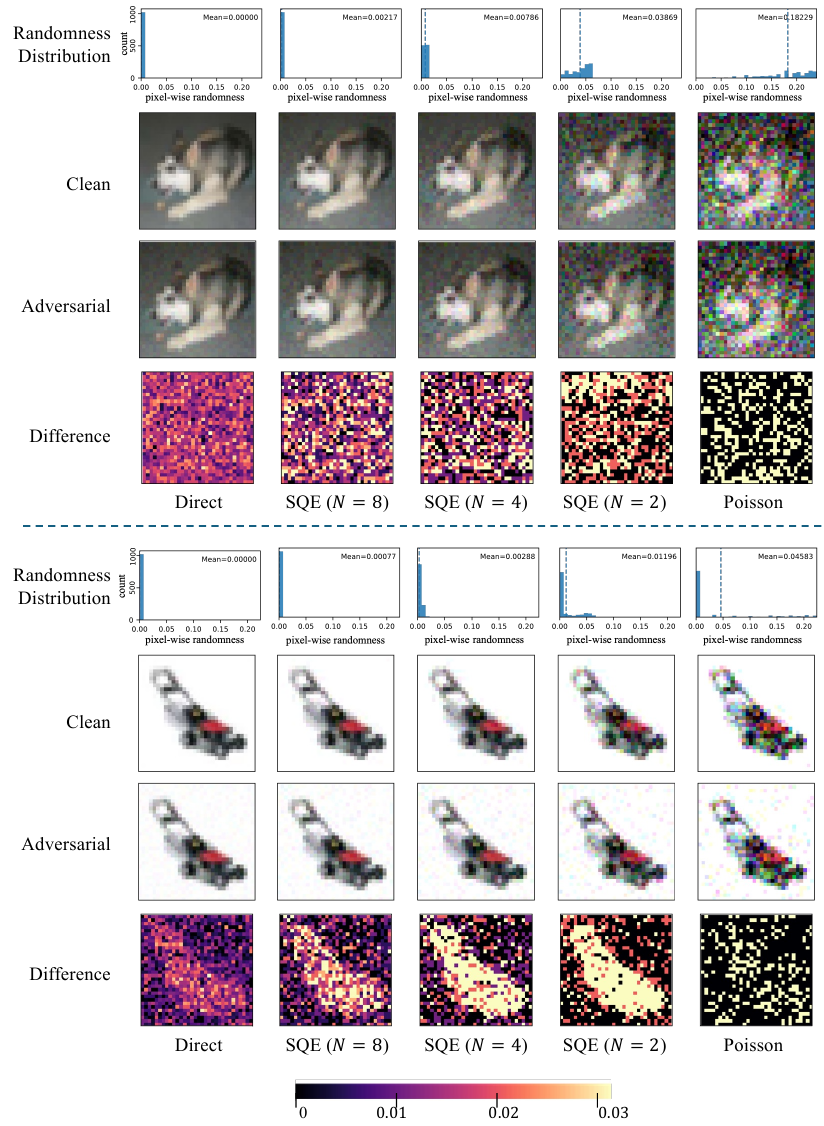}
    \caption{Visualization of clean images and their corresponding adversarial images generated by PGD10 after different encoding methods. For SQE and Poisson encoding, the encoded images are obtained by averaging the encoding results over eight time steps. For each encoding method, the first row shows the distribution of pixel-wise encoding randomness (Definition~1), where the dashed line and the value in the upper-right corner indicate the mean randomness across all pixels. The last row shows the absolute difference between the encoded clean and adversarial images, illustrating the effect of adversarial perturbations after encoding. Zoom in for a better view.}
    \label{fig:visual_appendix}
\end{figure*}

\section{Visualization Results}

In this section we provide additional visualizations that illustrate how the quantization scale $N$ affects the encoded images. Fig.~\ref{fig:visual_appendix} shows two groups of examples, within each group the four rows correspond to the pixel-wise encoding randomness, the encoded clean image, the encoded adversarial image, and the magnitude of their difference, which corresponds to the adversarial perturbation after encoding. Each visualized encoded image is obtained by averaging the encoding results over 8 time steps. There are three main findings.

First, as the quantization scale $N$ decreases, the distribution of pixel-wise encoding randomness gradually shifts to the right, indicating an increase in encoding randomness. This observation is consistent with Theorem~2, which shows that the quantization scale $N$ controls the degree of encoding randomness. For example, in the first example, the distribution is concentrated near zero, with a mean randomness only $0.002$ when $N=8$. As $N$ further decreases to $1$ (Poisson encoding), almost no pixels exhibit randomness close to zero, while most have randomness greater than $0.15$. Consequently, the mean randomness across all pixels increases to $0.182$, which is approximately 90 times that of $N=8$.

Second, the encoded clean image exhibit more noticeable noise as $N$ decreases. This is more clear in the first group of examples. The object stays clearly recognizable under direct encoding and $N=8$, is progressively corrupted as $N$ is reduced, and is barely identifiable at $N=1$, i.e., Poisson encoding. The noise caused by the encoding is most apparent in flat, low-texture regions such as the background. For example, when comparing the encoded clean images at $N=8$ and $N=2$, the noise in the background introduced by encoding becomes visibly stronger at $N=2$. This phenomenon is consistent with our theoretical analysis. A smaller $N$ introduces larger rounding noise due to stronger randomness. Such residual noise lowers the fidelity of the encoded clean input, which may explain the decrease in clean accuracy when $N$ is small.

Third, decreasing $N$ does not necessarily reduce the magnitude of the clean--adversarial difference, but instead changes its spatial distribution. Since stochastic quantization is unbiased, the expected difference between the encoded clean and adversarial images equals the input perturbation itself, making it largely independent of $N$. Thus, the mean intensity of the difference maps is approximately preserved across different encoders. However, when $N$ is large, the differences are small in magnitude but densely distributed across the image. For example, under direct encoding, the difference map in the second group is activated over almost the entire image. In contrast, with a smaller $N$, the differences become much sparser, with more pixels remaining unchanged and appearing black in the difference maps, as observed for quantization encoding with $N=2$ and Poisson encoding.

However, decreasing $N$ also tends to reduce the observable difference between the encoded clean and adversarial images, as measured by the intensity of the difference maps. When $N$ is large, the differences between the encoded clean and adversarial images have relatively small magnitudes but are densely distributed across the image. For example, in the first group of examples, the difference map after direct encoding is activated over nearly the entire image region. In contrast, as $N$ decreases, the differences become increasingly sparse, with more regions appearing black in the difference maps. This observation suggests that stronger encoding randomness may weaken the effect of adversarial perturbations in the encoded representation, thereby partially explaining the improved adversarial robustness.

\end{document}